\documentclass[runningheads]{llncs}
\usepackage[T1]{fontenc}
\usepackage{amsfonts}
\usepackage{amsmath}
\usepackage{graphicx}
\usepackage{subfigure}
\usepackage{algorithm}

\usepackage{algpseudocode}
\begin{document}

\title{A general-purpose method for applying Explainable AI\ 
for Anomaly Detection}
\titlerunning{Explained Anomalies}
%
\author{John Sipple\inst{1,2} \and
Abdou Youssef\inst{2}}
\authorrunning{Sipple, J., Youssef, A.}
%
\institute{Google, Mountain View, California, USA \and
The George Washington University, Washington DC, USA \\
\email{sipple@google.com}
\email{\{jsipple, ayoussef\}@gwu.edu}}
\maketitle              
\begin{abstract}
The need for explainable AI (XAI) is well established but relatively little has been published outside of the supervised learning paradigm. This paper focuses on a principled approach to applying explainability and interpretability to the task of unsupervised anomaly detection. We argue that explainability is principally an \emph{algorithmic} task and interpretability is principally a \emph{cognitive} task, and draw on insights from the cognitive sciences to propose a general-purpose method for practical diagnosis using explained anomalies. We define Attribution Error, and demonstrate, using real-world labeled datasets, that our method based on Integrated Gradients (IG) yields significantly lower attribution errors than alternative methods.

\end{abstract}
\title{A general-purpose method for applying Explainable AI \ 
to anomaly detection to enable the expert to discover, \ 
diagnose and treat device defects}

\section{Anomaly Detection and Interpretability}
\label{intro}
Imagine you are a technician that maintains a system with thousands of networked devices. Chances are high that at any given time, somewhere in the fleet there are faulty devices that require your attention. As a technician, you would like an anomaly detector to provide a rich explanation about the symptoms, allowing you to detect, diagnose, prioritize and fix faulty devices. This paper is motivated by the unmet need to provide the technician \textit{explained anomalies}. We combine insights from the cognitive sciences with explainable AI (XAI), and propose a general-purpose approach to aid the technician in detecting the anomaly and understanding the fault behind the anomaly. We show comparative results with sensor failure on a Variable Air Volume (VAV) device, and a fuel pressure failure on a general aviation aircraft engine.  



\emph{Anomaly Detection} (AD) is the machine learning task of detecting observations that do not conform to expected or normal behavior \cite{Chandola2009}, \cite{tax2002}, \cite{schoelkopf2001}, \cite{ruff2020}. AD is susceptible to both false positive and false negative errors \cite{Aggarwal2017}. Adjudicating false positive errors is often time consuming and may even make an AD solution ineffective. A nondescript anomaly score provides little insight to understanding what caused the anomaly and choosing the best treatment. We hypothesize that integrating XAI techniques with AD can reduce the human workload in processing a stream of anomalies.  

Relatively little about explainable AD has been published. An explainable AD method was developed by combining an autoencoder with the SHAP explainability method \cite{antwarg2020}. DIFFI \cite{Carletti2020} is an explainability technique for Isolation Forest (IF) \cite{Liu2008} that generates attributions based on dimensions that contribute significantly to isolating anomalous points and causing data imbalances by random data splits. MADI \cite{sipple2020a} combines a neural network classifier that distinguishes between observed and randomized samples, and uses Integrated Gradients (IG) \cite{Sundararajan2017}. The One-Class Deep Taylor Decomposition (OC-DTD), uses a One-Class Support Vector Machine (OC-SVM) to train (“neuralize”) a feed-forward neural network \cite{ruff2020}. Then, it applies Deep Taylor Decomposition \cite{bach2015} that yields variable attributions for anomalies \cite{kauffmann2020}. Explaining anomalies in image data, FCDD \cite{liznerski2021}, applies a convolutional upsampling technique to create an heatmap of anomalous regions detected by a neural network variant of OC-SVM.

\label{sec:contributions}
This work makes the following contributions to the field of explainable AD: \textbf{a.} creates a clear distinction between explainability and interpretability; \textbf{b.} proposes a methodology that leads to meaningful anomaly explanations combining the cognitive sciences with XAI methods; and \textbf{c.} defines Attribution Error and provides two real-world datasets for evaluating anomaly explanations.

\section{Conventions and Basic Definitions}
\label{sec:conv}
A \textit{system} is an observable, complex, and stateful process or object, which may be biological, physical, social, economic, etc. We define \textit{expert} broadly as a stakeholder of the system who applies a knowledge base to interpret the explanation to achieve a specific goal. Device $s$ in system $S$ generates periodic observations $x_{s,t}$ in $\mathbb{R}^D$ at time $t$ , where we omit subscripts $s$ and $t$ for simplicity. The baseline set, or normal points, are indicated with an apostrophe, $X'={x'_1, x'_2,  ...}$. An anomaly detector is a classifier model $F$ that scores each observation between $anomalous$ $(0)$ and $normal$ $(1)$: $F:\mathbb{R}^D\rightarrow{[0,1]}$. An \textit{explanation function}, $B:\mathbb{R}^D\rightarrow{[0,1]^D}$ maps an observation $x$ to an explanation, \textit{attribution}, or \textit{blame}, $B(x)=b$, where $b_d\in [0,1]$ and $\sum_{d=1}^{D}b_d\leq 1$. With respect to a specific baseline point $x'$, we shall include the baseline point as an argument, $B(x,x')$. An \textit{interpretation function} is a mapping from explanations to the powerset of discrete \emph{diagnostic conditions}: $I:[0,1]^D\rightarrow \mathcal{P}(C)$, where $C$ is a set of all diagnostic conditions known to the expert, $C={c_1,c_2, ...}$. 


\section{Explainability vs. Interpretability}
\label{sec:expl_vs_interp}
Recent works on XAI have proposed various definitions of explainability and interpretability. \cite{gilpin2018} describes explainability as models that are able \emph{to describe the internals of a system that is understandable to humans}. \cite{Lipton2016a} argues that interpretability reflects on \emph{trust, causality, transferability, informativeness, and fair and ethical decision making}. \cite{doshivelez2017} define interpretability as \emph{the ability to explain or to present in understandable terms to humans}. \cite{miller2021} equates interpretability with explainability as a mode in which an observer can understand a decision. According to \cite{Broniatowski} explainability is the model’s ability \emph{to provide a description of how a model’s outcome came to be, and interpretability refers to a human’s ability to make sense, or derive meaning, from a given stimulus so that the human can make a decision}. Similar to \cite{Broniatowski}, we propose that explainability and interpretability are two distinct ideas. The communication model with an information source, a message, a transmitter, a noisy channel, a receiver, and a destination \cite{ShannonWeaver49} provides a framework for describing the difference between explainability and interpretability, Table \ref{tbl:explainabilty-interpretability}. The information source are the devices that stream multidimensional observations. The transmitter is the anomaly detector that generates an anomaly score and an explanation. The receiver is the \emph{expert}, who applies her knowledge of the world to make judgments of the anomaly. Hence, \emph{Explainability} is the \emph{algorithmic task} of extracting an explanation from the model, and \emph{interpretability} is the \emph{cognitive task} of combining the explanation with the expert’s knowledge base, and forming a hypothesis. 
\begin{table}[t]
\caption{Differences between Explainability and Interpretability.}
\label{tbl:explainabilty-interpretability}
\setlength{\arrayrulewidth}{0.1mm}
\renewcommand{\arraystretch}{1.3}
\begin{center}
\begin{small}
\begin{tabular}{ |p{6.0cm}|p{6.0cm}| }
\hline
\textbf{Explainability}:\textit{Generate a prediction and explanation.} &\textbf{Interpretability}: \textit{Align the explanation with a knowledge base, and complete a task.}\\
\hline
Role of Transmitter  & Role of Receiver \\
\hline
Algorithmic Task & Cognitive Task\\
\hline
Expressed by Model & Generated and used by Expert  \\
\hline
Spatial, Temporal Correlation & Root cause and Causation \\
\hline
Observation to Attribution  & Attribution to Diagnostic Condition \\
\hline
Expresses symptoms & Diagnoses and treats\\
\hline
\end{tabular}
\end{small}
\end{center}
\end{table}
\subsection{Explainability}
\label{sec:explainability}
There are several ways to describe methods of XAI. \emph{Global methods} provide general explanations generally true about all model predictions. \emph{Local methods}  \cite{adadi2018}, \cite{samek2017} provide explanations for individual observations. Since they are specific to individual observations, \emph{local explainability} generally is more suitable for interpreting anomalies. Some local model explainability techniques, such as LIME \cite{ribeiro2016} and SHAP \cite{lundberg2017}, treat models as black boxes and perturb the inputs and observe the prediction and can be applied to any classifier model, and are called \emph{model-agnostic}. In contrast, \emph{model-specific} explainability techniques often place specific requirements on the model, such as the gradients or network architecture, that make them compatible. Examples include IG, Layer-wise Relevance Propagation (LRP), and DeepLIFT.  

\emph{Variable attribution}, or \emph{blame}, $b_d=B(x_d)$, quantifies the importance of the value on dimension $d$, $x_d$, in directionality (sign) and weight (magnitude) on a prediction. In non-linear models, these variables are local approximations of a hyperplane fit along the steepest gradient in feature space. Some local methods, such as IG \cite{Sundararajan2017} and Deep LIFT \cite{shrikumar2017}, require a neutral baseline point. 

\subsection{Interpretablity}
\label{sec:interpretablity}
Interpreting the model’s explanation requires the expert to align the observation, prediction, and explanation with the system’s context and a relevant knowledge base \cite{ricoeur1972}. To be successful, the expert must associate the model’s explanation with a particular cause or fault, and then decide to take a corrective action.  



\textbf{Anomaly Detection Expert} We propose there are three kindes of experts that use eplainable AI: \emph{physicus}, \emph{technicus}, and \emph{secularus}, described in Appendix \ref{appendix_expert_personas}. Technicus has a detailed knowledge base of the system, and aims to alter or influence the system in response to known conditions \cite{samek2017}. The explanations should describe the anomaly in its original state and enable the technicus to easily choose the most suitable treatment for the anomaly. The natural objective of XAI is to generate explanations that enable effective interpretation for the expert. Therefore, it is helpful to consider how the human mind recognizes, generalizes and categorizes, and reasons about causality in order to better design explanations. 

\textbf{Generalization.}\label{generalization} Empirical studies in cognitive sciences revealed that humans assign objects to categories based on some notion of distance or dissimilarity between multidimensional stimuli \cite{nosofsky1986}. One prominent generalization theory developed by \cite{shepard1957}, \cite{sheperd1987} suggested that distance can be well approximated with Euclidean L2 or City Block L1 distance measures, and stimuli are internalized in some internal metric space called \textit{psychological spaces}. Experiments with human subjects also indicated that similarity generally follows an exponential decay with distance. 

\textbf{Prototypes and Exemplars.}\label{sec:prototypes_exemplars} It is well established that when categorizing stimulus of a previously unseen object, humans will tend to use previously observed objects as a reference for making their decision. This gave rise to Prototype and Exemplar Theories \cite{navarro2007} \cite{lieto2017}, where the former suggests that humans remember an anchoring object (i.e., a prototype) when categorizing, and the latter suggests that humans tend to retain multiple examples in memory for a category or concept. 

\textbf{Contrastive Explanations.} \label{sec:constrastive_explanations} One of the most important conclusions from a review of explanations for AI drawn from the social sciences, was that good explanations are \emph{contrastive} \cite{miller2017}. Humans perform better in understanding explanations when presented with a contrastive, counterfactual event, and prefer explanations with \textit{why did event P happen, instead of counterfactual event Q}.

\section{An Approach for Interpreting Anomalies}
\label{sec:approach}
In this section, we combine AD, explainability and interpretability concepts. The proposed approach provides an explanation with both a contrastive baseline point $x'$ and a variable attribution or blame, $B(x,x')$. We assume that we have already trained a differentiable anomaly detector $F(x)=[0,1]$, from which gradient $\nabla{F}$ can be computed. The historic data for the unlabeled training dataset $X$ is assumed to be available, and is representative of new observations from system $S$. Each of the five steps is presented as a separate subsection next.

\subsection{Choose an exemplar baseline set}
As discussed in section \ref{sec:expl_vs_interp}, research in psychology suggests that humans retain and comprehend concepts with the support of exemplars, or reference points. We select baseline reference points from the training data that represent normal. The normal baseline sample from the training set have highest scores: $F(x)\approx1$. Normal observations are suitable candidates for  the baseline set. 

\begin{proposition}
\emph{If the training data is a statistically representative sample of the inference set, then points with the highest normal scores are suitable exemplars. }
\end{proposition}

In many complex systems, devices operate in different configurations (heating/cooling, idle/active, takeoff/landing, etc.), and distribute observations across multiple modes. Multiple modes occur when these device configurations generate probability mass functions with multiple local maxima. In cases where there are multiple distinct baseline modes, and when the number of normal points is unbalanced across modes, as illustrated with modes $A$ and $B$ in Figure \ref{fig:modes} (left), there is no guarantee that a baseline will represent all modes. Furthermore, it is usually not efficient to oversample baseline points from dense regions. For these reasons, we propose a clustering-based approach for selecting $N$ points randomly from each cluster, presented in Appendix \ref{appendix:select_baseline}.



\begin{figure}[htb]
  \centering
  \begin{subfigure}
  \centering
    \includegraphics[width=0.3\columnwidth]{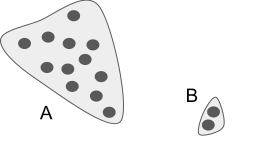}
  \end{subfigure}
  \hfill
  \begin{subfigure}
   \centering
    \includegraphics[width=0.5\columnwidth]{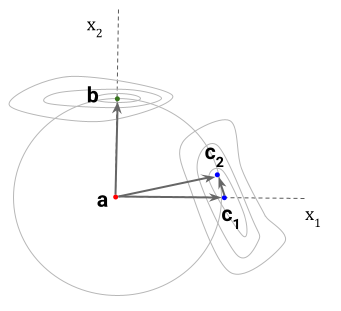}
  \end{subfigure}
  \caption{(Left) Multimodal distribution of the data, where mode $A$ has more observations than mode $B$. (Right) Selection of baseline points for anomaly $a$ with three equidistant point from two modes.}
  \label{fig:modes}
\end{figure}

\subsection{Choose a measure of dissimilarity}
\label{sec:minkowski_distance}
We must apply distance function for 1. selecting the baseline set, and 2. for selecting the baseline point for each anomaly. As discussed in Section \ref{sec:interpretablity}, based on empirical evidence on human categorization of variables are correlated in complex systems, adopt a Euclidean L2 or City Block L1 distance function. 

\begin{proposition}
\label{distance}
\emph{Euclidean L2 or City Block L1 are suitable for choosing a contrastive baseline point. }
\end{proposition}

\subsection{Select a baseline point} 
\label{sec:nearest_baseline}
Guided by Occam’s Razor, we seek the simplest, most efficient explanation. Transforming an anomalous point into any baseline point yields an explanation, regardless of its distance from the anomaly. Since each dimension has been normalized, an equal-length displacement in any direction incurs the same cost. Therefore, a baseline point from the nearest mode leads to the most efficient explanation. We apply explainability methods that use baselines as exemplars and choose the nearest exemplar, such as IG. The distance between each anomalous point and the baseline can be computed with the distance function, and the point with the minimum distance is chosen as the representative baseline point.

\begin{definition}
\label{equivalent_attributions}
Two baseline points $x'_1$, $x'_2$ result in \emph{equivalent attributions} if for any anomalous point $x$, the blame attributions $B(x,x'_1)$ and $B(x,x'_2)$ map to the same diagnostic condition, i.e., $I(B(x,x'_1))=I(B(x,x'_2))$.
\end{definition}

As will be shown in Theorem \ref{thm:equivalent_attributions} below, two baseline points drawn from the same mode and separated by a small perturbation result in the same interpretation, i.e., result in equivalent attributions. This justifies our next proposition:

\begin{proposition}
\label{small_perturbation}
A small perturbation in the baseline point yields an equivalent attribution. 
\end{proposition}

Before the formalization and proof of this proposition, we provide a simple illustration. Consider the two-dimensional case shown in Figure \ref{fig:modes} (right) with anomalous observation $a$. Suppose that three possible baseline points were selected: point $b$ from one mode, and points $c_1$, $c_2$ are from a different mode, and $F(a)=0$, $F(b)=F(c_1)=F(c_2)=1$. 

The attribution of $a$ with respect to $b$ will be the vector $[0, 1]$, but the attribution of $a$ with respect to $c_1$ will be $[1, 0]$. The attribution of $a$ with respect to $c_2$ will be $[1,0]$, i.e., nearly equal to the attribution with respect to $c_1$. While the attributions on $b$ and $c_1$, $c_2$, are different, as long as all the points are equidistant from $a$ and have the same classification score, they are valid attributions. 


\begin{theorem}
\label{thm:equivalent_attributions}
The reference points $c_1$ and  $c_2$ will generate \emph{equivalent attributions} if the following conditions are met: (a) the points are close $c_1\approx c_2$, (b) equidistant from anomaly $a$, and (c) have the same classification score $F(c_1)=F(c_2)$. That is,  under those conditions, the difference of attributions on each dimension will tend to zero as $c_1$ and  $c_2$ become infinitesimally close,

$\lim_{\|c_1-c_2\|\to 0} \left( IG_d(a,c_1)-IG_d(a,c_2)\right)=0$.
\end{theorem}

Proof is provided in Appendix \ref{sec:proof_equiv_attr_appendix}.

Theorem \ref{thm:equivalent_attributions} implies that a sparse sampling of reference points in the neighborhood $x \in X: F(x) \approx 1$ as baseline points can be used to represent normal modes.

\subsection{Choose a path from the anomaly to the baseline}
According to \cite{Sundararajan2017}, any path connecting the anomaly $x$ and its baseline point $x'$ is valid. However, to be consistent with findings from cognitive tests performed on human cohorts, described in Section \ref{sec:interpretablity}, it is most appropriate to prefer the L1 path (city block) or L2 path (Euclidean). According to the cognitive tests, human subjects preferred the L2 path when significant correlation is manifested in two or more dimensions.

\begin{proposition}
\label{prop:path}
\emph{Both L1 and L2 paths connecting the anomaly and baseline are  acceptable for interpretability.}
\end{proposition}

\subsection{Apply an explanation function with the baseline}
\label{sec:explanation_with_baseline}
As defined in Section \ref{sec:conv}, an explanation function accepts an AD model $F$, an observation $x$, and the model’s score of the observation $F(x)$, and returns a variable attribution on the input data, $B(x)$. In this section, we propose multiple desirable properties of an explanation function, and then evaluate various methods for suitability. First, an explanation function must be \emph{contrastive}, as defined by \cite{miller2017}. In this work, we extend the definition with an approach that provides constructive explanations. In AD, observation $x$ is anomalous, so the suitable explanation would be based on a contrastive normal point $x'$, selected as discussed in \ref{sec:nearest_baseline} from the exemplar set.  Second, an explanation function must be \emph{conservative} (as in Definition 2, \cite{montavon2015})  or, equivalently, \emph{complete} (as in the the Completeness Axiom of \cite{Sundararajan2017} or Efficiency Axiom in \cite{sundararajan20b}, where the sum of attributions equals the difference in classification scores between the baseline and observed points: $\sum_{d \in D}B_d(x,x')=F(x')-F(x)$. Third, an explanation function for local explanations must satisfy \emph{sensitivity}, first in the sense of \cite{Sundararajan2017}, which states that any dimension that changes the prediction when its value is altered should also have a non-zero attribution. Further, any dimension that induces no change to the prediction when perturbed (i.e., dummy variable) should have a zero attribution. Fourth, the explanation should satisfy \emph{proportionality}\footnote{Proportionality in this work is different from Proportionality defined in \cite{Sundararajan2017}, since proportionality in the latter refers to a condition under which the dimensional components of the distance between a baseline point and an observation is proportional to the attribution.}; this is elaborated next.

\begin{definition}
\label{def:stronger_influence}
Given a path $P$ from $x$ to $x’$, dimension $u$ has a \emph{stronger influence} on $F$ than dimension $v$ over $P$ if the average (or overall) rate of change of $F$ along dimension $u$ is greater than the average (or overall) rate of change of $F$ along dimension $v$.
\end{definition}

If $P$ is the path defined by $z=x+ \alpha(x'-x)$ for $\alpha \in [0, 1]$, between $x$ and $x'$, then the overall rate of change of $F$ along dimension $u$ over path $P$ is $IG_{u}(x, x') \equiv \int_0^1(x'-x)_u \nabla_u F(x + \alpha(x'-x))\,d \alpha$. This is justified by taking $G(\alpha)=F(z)=F(x+ \alpha (x'-x))$,  and observing that
$$G'(\alpha)=\sum_u \nabla_u F \times \frac{\partial z_u}{\partial \alpha} = \sum_u \nabla_u F \times (x' - x)_u$$
that is, the sum of dimensional rates of change of $F$ over path $P$. Therefore, the instantaneous rate of change of $F$ at a point $z=x+ \alpha (x'-x)$ along dimension $u$ is $\nabla_u F \times (x'-x)_u$, and thus the overall rate of change along dimension $u$ is the \emph{integral} of the instantaneous rate of change of $F$ over the path. 

Therefore, Definition \ref{def:stronger_influence} becomes: Given a path $P=(x+\alpha(x'-x))$ from $x$ to $x'$, dimension $u$ has a \emph{stronger influence} on $F$ than dimension $v$ over $P$ if 
\begin{equation}
\label{eq:ig_stronger_influence}
\begin{split}
\int_0^1(x'-x)_u \nabla_u F(x + \alpha(x'-x))\,d\alpha \\
 > \int_0^1(x'-x)_v \nabla_v F(x + \alpha(x'-x))\,d\alpha
\end{split}
\end{equation}
that is, dimension $u$ has a \emph{stronger influence} on $F$ than dimension $v$ over $P$ if $IG_u(x,x')>IG_v(x,x')$.  

\begin{definition}
\label{def:proportional}
An explanation function $B(x,x')$ is \emph{proportional} on observation $x \in \mathbb{R}^D$ if for any dimension $u$ that has a stronger influence on $F$ than another dimension $v$ over a path $P$ from a point $x$ to a point $x'$,  then $B_u(x,x')>B_v(x,x')$. 
\end{definition}

In other words, dimensions that have a stronger influence on the anomaly prediction $F$ should have greater attributions than dimensions that have a weaker influence. Proportionality enables ranking and thresholding of dimensions.

We evaluated LIME \cite{ribeiro2016}, SHAP \cite{lundberg2017}, Layerwise Relevance Propagation (LRP) \cite{bach2015} \cite{binder2016}, Deep Taylor Decomposition (DTD) \cite{kauffmann2020}, One Class Deep Taylor Decomposition (OC-DTD) \cite{kauffmann2020} and Integrated Gradients (IG) against contrastive, conservative, sensitivity and proportionality desiderata. Details and jstification are provided in Appendix \ref{appendix:xai_disidiratum}. LIME, LRP, and most implementations of SHAP are not contrastive \cite{merrick2019}. DTD, OC-DTD make approximate the gradient as a linaerly, and are not guaranteed to satisfy proportionality. IG does not specify a fully defined method for choosing a baseline point. Instead, its authors provided general guidelines for selecting a good baseline, where a good baseline point should have nearly opposite scores as the point of interest (i.e., if $F(x) \approx 1$, then $F(x') \approx 0$, or vice-versa). Therefore, IG provides a contrastive explanation against the set of all baseline points that have significantly contrastive scores. Because it applies a path integral between $x$ and $x'$, IG accurately approximates the gradient, is conservative, and meets sensitivity criteria \cite{Sundararajan2017}. In addition, by computing the path gradient for each dimension, IG also satisfies proportionality (as shown in Theorem \ref{thm:ig_proportionality}). Therefore, because it satisfies each desideratum, we propose IG is a suitable explanation function for AD.

\begin{theorem}
\label{thm:ig_proportionality}
Integrated Gradients satisfies proportionality, as defined in Definition \ref{def:proportional}.
\end{theorem}
Proof is provided in Appendix \ref{appendix:proof_proportionality}.


\section{Experiments}
\label{experiments}
In this section, we fix the anomaly detector $F$ and compare the proposes approach, based on IG against two alternative explainability methods, SHAP and LIME.

\textbf{Attribution Labels} To our knowledge, there are no standard datasets curated for evaluating blame attributions for AD; therefore, we adopt the following labeling method. In addition to the familiar true binary anomaly class label, $A(x) = [anomalous,not\,anomalous]$, we append a vector $\beta(x)$ of dimension $D$, where $\beta_d(x) = 1/n_A$ if $x_d$ is one of $n_A$ relevant, explanatory dimensions of the defect selected by expert technicians for fault diagnosis, and $0$ otherwise; with $\sum_{d \in D}\beta_d(x)=1$.\footnote{We apply equal weighting to each relevant dimension in $\beta(x)$, because our technician labelers have found it impractical to to assign relative importance weights/preference to the relevant dimensions.} To score the attribution for method $i$, $B^{(i)}(x)$, against the label $\beta(x)$ for anomalous observation $x$, we compute the \textbf{Blame Attribution Error} as the mean absolute difference: $\epsilon^{(i)}(x) = \sum_{d \in D}\left| B^{(i)}_d(x) - \beta_d(x)\right|/|D|$ for every point $x$ where $A(x)=anomalous$. 

\textbf{VAV dataset} (406 anomalous observations) A Variable Air Volume (VAV) device provides ventilation and heating using airflow dampers and a hot water heat exchanger, and are commonly installed in commercial office spaces. There are 12 dimensions reported by the VAV, including temperature and airflow setpoints and measurements. Due to an intermittent failure of its the temperature sensor, during normal business hours the VAV's airflow sensor reports values significantly higher than the airflow setpoint. The hot water valve is consistently open, reporting around 100\%, and the zone air temperature is below its setpoint. We trained one anomaly detector (AUC=0.99) based on \cite{sipple2020a} using unlabeled historic observations from a cohort of 203 similar VAVs.

\textbf{Aircraft dataset} (448 anomalous observations) This dataset contains 16 real-valued dimensions extracted from a small general aviation aircraft flight data recorder, on which the fuel pump partially failed in flight. During normal operations, the fuel pressure ranges between 50 and 80 psi, with only short bursts above 80. During the failure, the fuel pressure achieved pressures between 80 and 120 psi. We trained an anomaly detector based on \cite{sipple2020a} (AUC=0.94) with unlabeled earlier flight recordings of the same aircraft.

We compare the proposed method using IG, with two model-agnostic explanation functions, SHAP and LIME and compute the average and standard deviation of the explanation errors, shown in Table \ref{tbl:results-table}. In both datasets, IG yielded significantly lower attribution errors than SHAP and LIME. LIME appears to be vulnerable to larger because it of the local approximation around $x$, and in many cases the gradient in the neighborhood of the anomaly may be very small and noisy resulting in unstable attributions. We hypothesize that IG achieves the lowest anomaly attribution error because it traverses the entire gradient from the anomaly to the nearest baseline point avoiding local approximations.

\begin{table}[t]
\caption{Mean and Standard Deviations of attribution error values  \% for labeled anomaly explanation datasets. Highlighted values are the top-scoring detectors based on a 5\% significance threshold (Mann-Whitney U test).}
\label{tbl:results-table}
\begin{center}
\begin{sc}
\begin{tabular}{llll}
& IG & SHAP & LIME \\
\hline
Aircraft &\textbf{6.0$\pm$1.9}   &7.4$\pm$2.8      & 12.8$\pm$0.2 \\
VAV &\textbf{7.5$\pm$2.8} &13.1$\pm$1.0 &13.8$\pm$0.1  \\

\hline
\end{tabular}
\end{sc}

\end{center}
\vskip -0.1in
\end{table}

\section{Discussion and Future Work}
\label{sec:discussion}
This paper considers how XAI can be applied to AD to provide insightful, contrastive explanations that enable the expert to understand an anomalous observation. We proposed that explainability is an \emph{algorithmic} transformation from prediction to the explanation, and interpretability is a \emph{cognitive} transformation from explanation to the diagnostic condition. We have considered numerous explainability methods, and reviewed what aspects are important for a good explanation, inspired by the social sciences. We propose a novel method of choosing baseline points that enable contrastive explanations. We propose a method for evaluating blame attributions and show that IG has the lowest attrubution error in two real-world datasets.

While this paper proposes an approach to scoring blame attributions, additional work is needed in evaluating the accuracy and importance of a contrastive normal baseline.  It will be beneficial to demonstrate the general utility of this relationship across different types of systems and AD models, such as Appendix \ref{appendix_illustration}. With advances in wearable technology, is it  possible for an explainable anomaly detector to predict illnesses by simply wearing a few commodity sensors? 

\subsubsection{Acknowledgements} The authors would like to thank Klaus-Robert Müller and Ankur Taly for instructive and practical advice and for their detailed technical reviews, and the anonymous reviewers for identifying gaps and suggesting improvements.

\newpage
\bibliography{main}
\bibliographystyle{splncs04}

\newpage
\appendix
\onecolumn
\section{Expert personas for Explainable AI.}
\label{appendix_expert_personas}

As with explainability, interpretability can be subdivided by their distinguishing characteristics. While the model is responsible for generating an explanation, the expert is responsible for forming an interpretation. The important characteristics can be largely subdivided into the expert’s knowledge base, the expert’s goals, and the expert’s action space. Table \ref{tbl:expert_personas} shows how we  expand on the foundational definitions by \cite{ras2018} that distinguished between the expert and the secularus. We propose each user of the explanation is an expert of their own kind, and can be broadly classified into three personas: physicus, technicus, and secularus\footnote{We prefer the Latin terms over scientist, technician, and layperson to encompass a broad range of experts without specifying the type of system. For example, technicus is meant to include both technicians that fix devices and medical professionals that treat patients.}. Furthermore, we distinguish between a system-oriented expert and a model-oriented expert.

\emph{Physicus} interprets the explanation to discover natural laws, or uncover a deeper knowledge of the phenomenology of the observed system that produced the interesting observation, and hence, the expert’s goal is epistemological. \cite{samek2017} refers to the goal of the physicus as learning from the system. This expert’s knowledge base of the system is extensive, and is focused on discovery and expanding the existing knowledge base of the system. Like an astronomer, climatologist, or macro-economist, the system-oriented physicus may not have a direct influence on the system. In contrast, the model-oriented expert seeks to discover new knowledge of the model, perhaps discovering new methods and techniques that will yield better, more accurate, insightful, fair, reliable, and transparent models. 

Like physicus, \emph{technicus} has a detailed knowledge base of the system, but unlike physicus, the goal of technicus is to directly alter or influence the system in response to known conditions, and may choose a response, i.e. the goal is to improve the system (or model) \cite{samek2017}. In order to choose the response to the model’s explanation, technicus must choose an action or treatment in response to the explanation. Technicus is able to perform an effective diagnosis and treatment when the observed system is well understood and mostly static, such that past successful treatments can be repeated to yield similarly beneficial results. The model-oriented technicus trains, validates, and deploys the model without attempting to alter the system; the system oriented technicus uses the explanation to alter the system.

\emph{Secularus} is an observer with the least detailed knowledge of the system or model compared to either the technicus or the physicus, and has a limited capability to interpret the explanation. Secularus’s perspective is ego-centric: while not able to alter the system, secularus can choose a response that is beneficial to his interests. For example, secularus may interpret an explanation from a model observing the stock market, and may choose a response to buy or sell. Secularus may consider the current cloud cover, temperature, dew point, barometric pressure and decide to leave home with an umbrella, or may determine whether the model outputs are compliant with legislation \cite{samek2017}. 

Each system-oriented expert would like to know how the model’s decisions and explanations evolve over time. The physicus establishes governing laws that predict a future outcome by observing the prediction and explanation, the technicus anticipates a favorable change in the system following a treatment, and the secularus may choose their next action based on the explanation. In each case, the change in explanation over time plays a critical role. 

\begin{table}[t]
\caption{Three types of experts and how they interpret model explanations.}
\label{tbl:expert_personas}
\setlength{\arrayrulewidth}{0.1mm}

\renewcommand{\arraystretch}{1.3}
\begin{center}
\begin{small}
\begin{tabular}{ |p{3cm}|p{4.0cm}|p{4.0cm}|p{4.0cm}| }
\hline
\multicolumn{1}{|c|}{ } & \multicolumn{3}{|c|}{\textbf{Expert}} \\
\hline
 & \textbf{Physicus} & \textbf{Technicus} & \textbf{Secularus} \\
 & \emph{to discover} & \emph{to diagnose} & \emph{to decide} \\
\hline 
\textbf{System-oriented} 
& Researcher in econmocs, astronomy, biology, climatology...  
\begin{itemize} 
    \item[] \emph{"What makes the system exhibit this behavior?"}
\end{itemize}
& Medical doctor, automobile technician, power plant operator, mechanic, traffic monitor...       
\begin{itemize} 
\item[] \emph{“What went wrong and how can I fix the system?”}   
\item[] \emph{“Has the treatment improved the system’s condition?”} 
\end{itemize}
& Day trader, investor, medical patient, criminologist/police, driver, homeowner... 
\begin{itemize} 
\item[] \emph{“Should I buy shares in XYZ”}    
\item[] \emph{“Am I compliant with laws and regulations?”} 
\end{itemize}
\\
\hline
\textbf{Model-oriented} 
& AI researcher, Computer Scientist...
\begin{itemize} 
\item[] \emph{“Why is Deep Learning architecture X better than Y?”}
\end{itemize}
& AI practitioner, Software Engineer...
\begin{itemize} 
\item[] \emph{“When I deploy the model to production, will it be accurate, fair, trustworthy and reliable?”}
\end{itemize} 
& Law enforcement professional, hiring manager, passenger of self-driving vehicle... 
\begin{itemize} 
\item[]  \emph{“Do I trust the model enough to influence my decision?”}  
\item[] \emph{“Is my legal right to an explanation guaranteed?”}.  
\item[] \emph{“Does the model discriminate against any group?”} 
\end{itemize}
\\
\hline

\end{tabular}
\end{small}
\end{center}
\end{table}

\section{A clustering-based algorithm to selecting exemplar baseline set.}
\label{appendix:select_baseline}
Algorithm \ref{alg:select_baseline} provides details of the proposed method of selecting a baseline, given the training sample $X$, an anomaly detection classifier $F$, the maximum number $n$ of points per mode, and a tolerance threshold, $\epsilon$. 

While any clustering algorithm is suitable, density-based clustering such as DBScan \cite{ester1996} works well in practice since (1) it does not require the number of clusters to be specified,  (2) it performs well clustering arbitrarily shaped modes, and (3) it is robust to noisy data.

The baseline set $X'$ may then be considered the exemplar set and used as contrastive, normal points.

\begin{algorithm}
\caption{Select Multimodal Baseline}\label{alg:select_baseline}
\begin{algorithmic}[select_baseline]
\Procedure{Select Multimodal Baseline}{$X, F, n, \epsilon$} \Comment{normalized training set $X$, anomaly detector $F$, max points per node $n$, tolerance threshold  $\epsilon$}
  
  \State $\hat{X}=X:F(x \in X)>1-\epsilon$   \Comment{Choose observations with nearly normal scores.}
  \State $H = $\textbf{Cluster} $(\hat{X}, d_{Euclidean})$ \Comment{Cluster the points using L2 distance function.}
  \State $X'= $\textbf{Choose}$(n, H)$ \Comment{Randomly select up to $n$ points from each cluster in $H$.}
  \State \textbf{return} baseline set $X'$
\EndProcedure
\end{algorithmic}
\end{algorithm}


A side-by-side comparison of pair plots between the naive approach, as proposed in \cite{sipple2020a}, and the proposed mode clustering baseline selection algorithm is shown in Figure \ref{fig:selection_comparison}. The plot indicates that the proposed algorithm creates a baseline  sample that is less than half the size of the naive sampling strategy, yet provides much better coverage of the high-confidence region. 

\begin{figure}[htb]
  \centering
  \begin{minipage}[c]{0.48\columnwidth}
    \includegraphics[width=\columnwidth]{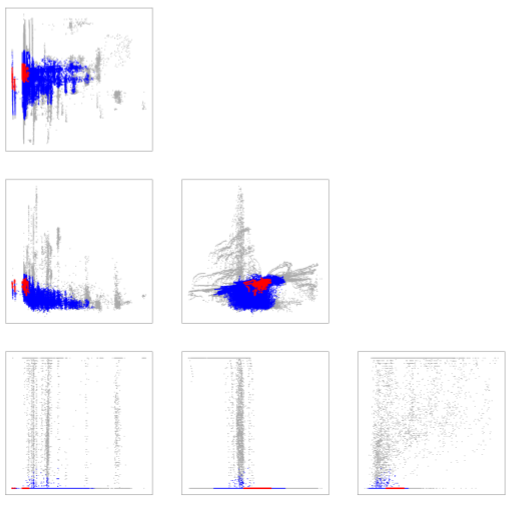}
  \end{minipage}
  \begin{minipage}[c]{0.48\columnwidth}
    \includegraphics[width=\columnwidth]{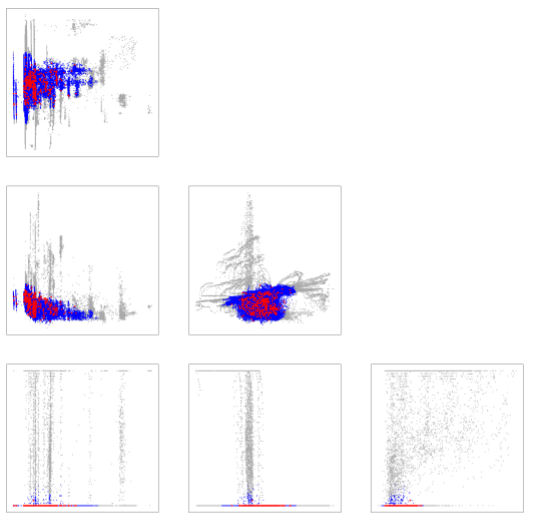}
  \end{minipage}
  \caption{(best viewed in color). Pairwise scatter plots of three-dimensional observations obtained from 36 VAV devices (discharge air temperature, zone air temperature, heating water valve percentage). Grey points represent low-confidence points with anomaly score $F(x)<0.9$, and are not eligible for baseline points; conversely, blue points are eligible to be chosen as baseline points with $F(x)\geq0.9$. The red points were selected as the baseline sample using (left) Naive sampling approach, and (b) proposed clustering-based approach. Note that (a) and (b) had the same observations $X$ and same anomaly detection classifier $F$. Even with a baseline size less than half of the naive baseline sample, the proposed mode-clustering approach provides much better coverage of the high-confidence regions. }
  \label{fig:selection_comparison}
\end{figure}

\section{Proofs}






\subsection{Proof for Equivalent Attributions}
\label{sec:proof_equiv_attr_appendix}
\begin{theorem}
\label{thm:equivalent_attributions_appendix}
The reference points $c_1$ and  $c_2$ will generate \emph{equivalent attributions} if the following conditions are met: (a) the points are close $c_1\approx c_2$, (b) equidistant from anomaly $a$, and (c) have the same classification score $F(c_1)=F(c_2)$. That is,  under those conditions, the difference of attributions on each dimension will tend to zero as $c_1$ and  $c_2$ become infinitesimally close, 

$\lim_{\|c_1-c_2\|\to 0} \left( IG_d(a,c_1)-IG_d(a,c_2)\right) =0$.
\end{theorem}

\begin{proof}
Explicitly define IG as a function of two points, anomaly $a$ and baseline $b$ for dimension $d$: $IG_d(a, b)=(b-a)_d\int_a^b\nabla_dF(x)\,d\alpha$. We’ll drop the $d$ subscript for simplicity, but dimensionality is implied. Let $\delta(c_1,c_2)=IG_d(a, c_1)-IG_d(a, c_2)$ be the difference in attribution with respect to anomaly $a$ on dimension $d$ for two points $c_1$ and $c_2$: 

\begin{equation} \label{eq:dif_attrib}
\begin{split}
\delta_{d}(c_1,c_2)  = (c_1-a)_{d}\int_a^{c_1}\nabla_{d}F(x)\,d\alpha -(c_2-a)_d\int_a^{c_2}\nabla_{d}F(x)\,d\alpha
\end{split}
\end{equation}

Apply the standard property of chained integrals to \ref{eq:dif_attrib}: $\int_u^wf(x)\,d\alpha=\int_u^vf(x)\,d\alpha+\int_v^wf(x)\,d\alpha$

\begin{equation}
\begin{split}
\delta_d(c_1, c_2) = (c_1-a)_d\left[ \int_a^{c_2}\nabla_dF(x)\,d\alpha + \int_{c_2}^{c_1}\nabla_dF(x)\,d\alpha\right]   -(c_2-a)_d\int_a^{c_2}\nabla_dF(x)\,d\alpha
\end{split}
\end{equation}

Rearrange terms:

\begin{equation}
\begin{split}
\delta_d(c_1, c_2) = (c_1-c_2)_d\int_{x=a}^{x=c_2}\nabla_dF(x)\,d\alpha  + (c_1-a)_d\int_{x=c_2}^{x=c_1}\nabla_dF(x)\,d\alpha 
\end{split}
\end{equation}

The gradient $\nabla F$ is assumed bounded inside the region between $a$ and $c_2$, that is, there is a constant $M$  where $|\nabla_{d}F|\leq M$ inside the region between $a$ and $c_2$. 

Therefore, 
\begin{equation}
\begin{split}
\left|\int_{a}^{c_2} \nabla F(x)\,d \alpha\right| \leq \int_{a}^{c_2}\left|\nabla F(x)\right|\,d\alpha \leq \int_a^{c_2}M\,d\alpha=M(c_2-a)  
\end{split}
\end{equation}
Hence,
\begin{equation}
\begin{split}
\left|(c_1 - c_2)_{d}\int_{a}^{c_2}\nabla_d F(x)\,d\alpha\right|  \leq \left|(c_1- c_2)M(c_2-a)_d\right| \rightarrow 0
\end{split}
\end{equation}

Similarly, $\int_{c_2}^{c_1} \nabla_d F(x)\, d\alpha \rightarrow 0$ as $\|c_1-c_2\|\rightarrow 0$ because $|\nabla_d F|$ is bounded in the region between $c_1$ and $c_2$. 

This implies that $$(c_1-a)_d\int_{c_2}^{c_1}\nabla_d F(x)\,d\alpha \rightarrow 0$$ as $\|c_1-c_2\| \rightarrow 0$.

It follows that $\delta_d(c_1, c_2) \rightarrow 0$ as $\|c_1-c_2\| \rightarrow 0$, i.e., $$\lim_{\|c_1-c_2\| \rightarrow 0} \left(IG_d(a,c_1)-IG_d(a,c_2)\right)=0$$
\end{proof}

\subsection{Proof of Proportionality}
\label{appendix:proof_proportionality}
\begin{theorem}
\label{thm:ig_proportionality_proof}
IG satisfies proportionality, as defined in Definition \ref{def:proportional}.
\end{theorem}

\begin{proof}
Consider an anomalous observation $x$ (i.e., $F(x)$ is near $0$), and a normal baseline point $x'$. Let $u$ and $v$ be two dimensions where dimension $u$ has a stronger influence on $F$ than dimension $v$ over path $P(\alpha)=x+\alpha(x'-x)$ from $x$ to $x'$, we need to prove that $B_u(x)>B_v(x)$.

By definition of IG, $B_u(x,x')$ with respect to exemplar $x’$ is $IG_u(x,x')$, and $B_v(x,x')$ is $IG_v(x,x')$. As indicated Equation \ref{eq:ig_stronger_influence}, the fact that dimension $u$ has a stronger influence on $F$ than dimension $v$ means that $IG_u(x,x')>IG_v(x,x')$, which trivially implies that $B_u(x,x')>B_v(x,x')$.   
\end{proof}
\textbf{Remark}: The above derivations give a new insight into IG, i.e., IG defines the blame/attribution to each dimension/feature $u$ to be the overall rate of change of the classifier $F$ along that dimension/feature over a path from an anomaly $x$ to an baseline point $x'$.

\section{Evaluation of Explainability Methods for Anomaly Detection}
\label{appendix:xai_disidiratum}
We consider which explanation functions satisfy contrastive, conservative, sensitivity, and proportionality desiderata.  
All explanation functions that do not apply a contrastive baseline point, such as LIME, SHAP, and LRP, do not meet the contrastive desideratum, and are eliminated. It is worth mentioning that recent work has adapted these approaches to contrastive baselines. Such as, \cite{merrick2019} proposed reformulating SHAP with blended coalition sets from multiple baseline points to form an averaged attribution. The authors clarify that when using multiple baseline points one must attend to attribution errors when the data distribution is multimodal, and propose techniques to group the baseline points by clustering them by their attribution. Additional investigation may be required to determine how well this approach works under arbitrary multimodal distributions.

First-order Taylor Decomposition techniques \cite{baehrens2010a}, \cite{montavon2015}, apply a Taylor expansion of the gradients from a root point $x’$: $F(x')-F(x)=\frac{\partial F}{\partial x}|_{x'}(x'-x)+\epsilon$, where $\epsilon \approx 0$ denotes the second- and higher-order terms. Because these techniques apply a linear approximation for the gradient, they are only suitable for comparing points in close proximity. Since the observed point and the contrastive baseline points are on extreme ends of the classification range, (i.e., $F(x') \approx 1$ and $F(x) \approx 0$), the gradient approximation at either point is, in general, not representative of the gradient along the path between $x$ and $x'$, and thus first-order Taylor approximations do not satisfy sensitivity and proportionality.

Deep Taylor Decomposition (DTD) \cite{montavon2015} is an extension of the first-order Taylor approximation for deep neural networks, where the first-order approximation is applied for each neuron, progressively from the output node to the input nodes. DTD is conservative and also applies a positivity constraint. However, its authors select a singular baseline (root) point at the origin by applying the positivity constraint. Therefore, by specifying the root point, no exemplar set is applied, and therefore, DTD is not contrastive\footnote{Even though DTD approximates the decision surface, it is not clear how a contrastive explanation may be formed.}.

A notable application of DTD to anomaly detection also worth considering is One-Class Deep Taylor Decomposition (OC-DTD) \cite{kauffmann2020}. The authors proposed first training a One-Class SVM as an anomaly detector, then \emph{neuralizing} the OC-SVM into a neural network, and then applying DTD for  variable attribution. The authors proposed a relevance model to approximate the layerwise attributions (relevances). The proposed relevance model computes a power-raised distance from the support vectors, rather than contrasting the observed point with an exemplar set. Since the basis of comparison are the support vectors and not an exemplar, OC-DTD also is not contrastive.

Different from DTD and OC-DTD, IG \cite{Sundararajan2017} does not specify a fully defined method for choosing a baseline point. Instead, its authors provided general guidelines for selecting a good baseline, where a good baseline point should have nearly opposite scores as the point of interest (i.e., if $F(x) \approx 1$, then $F(x') \approx 0$, or vice-versa). Therefore, IG provides a contrastive explanation against the set of all baseline points that have significantly contrastive scores. Because it applies a path integral between $x$ and $x'$, IG accurately approximates the gradient, is conservative, and meets sensitivity criteria. In addition, by computing the path gradient for each dimension, IG also satisfies proportionality (as shown in Theorem \ref{thm:ig_proportionality_proof}). Therefore, because it is contrastive, conservative, and satisfies sensitivity and proportionality, we propose selecting IG as the contrastive explanation function.

\section{A practical illustration of interpretable anomaly detection}
\label{appendix_illustration}
We illustrate explained anomaly detection with the Variable Air Volume (VAV) device, illustrated in Figure \ref{fig:vav}. The VAV maintains office room temperature $T_{room}$ above heating and below cooling setpoints, $T_{heating}$, $T_{cooling}$. Fresh air at temperature $T_{air,in}$ flows through the VAV. Hot water circulates through a heat exchanger to heat the air, $T_{air,out}$. Air flow is controlled with a mechanical damper, and water flow is regulated with a mechanical valve.

\begin{figure}[htb]
  \centering
    \includegraphics[width=\columnwidth]{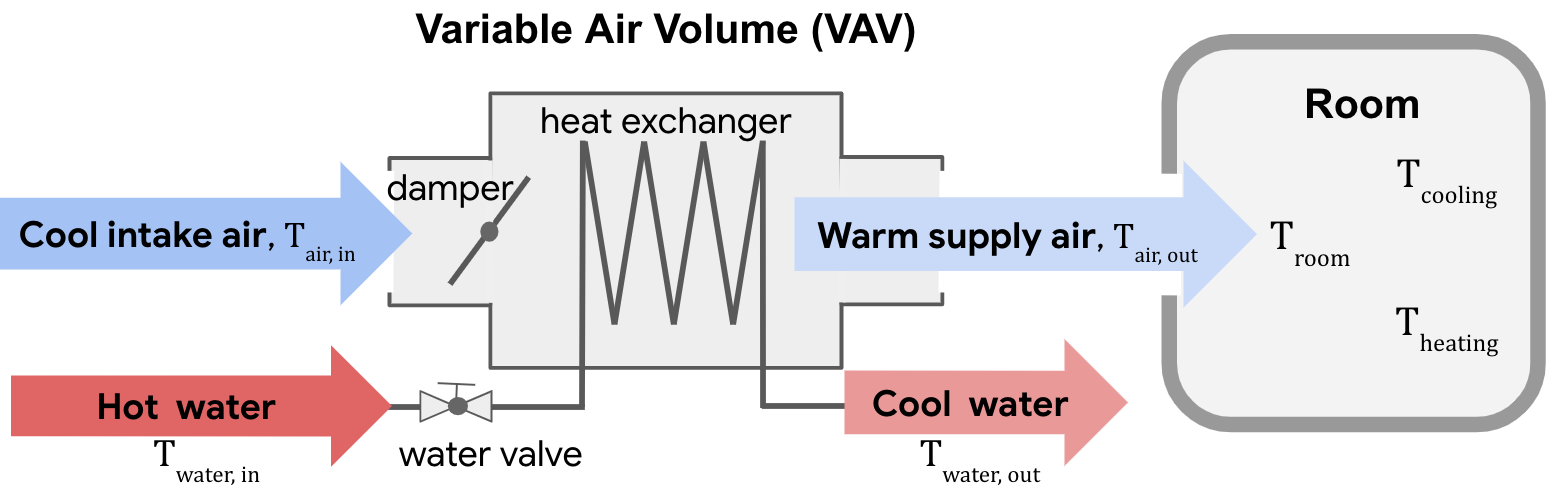}
  \caption{Schematic of a Variable Air Volume (VAV) device. }
  \label{fig:vav}
\end{figure}

\label{sec:root_cause_interpretation}
\textbf{Anomaly Diagnosis.} We trained a neural network with negative sampling anomaly detector \cite{sipple2020a} on 50 identical VAVs, and followed the approach from Section \ref{sec:approach} to generate explanations. The attribution time series of an actual water flow failure is shown in Figure \ref{fig:timeseries}.  Initially, the device exhibits normal behavior as indicated by the anomaly score. However, at 11:40 am the supply water temperature drops, causing VAV’s anomaly score to increase. The attribution is assigned to the heating water valve, which responds by opening beyond its normal condition. At 12:00, as the zone air temperature starts to cool, the attributions change indicating that the supply air flow rate changes to compensate for the insufficient water temperature.

\begin{figure}[htb]
  \centering
    \includegraphics[width=\columnwidth]{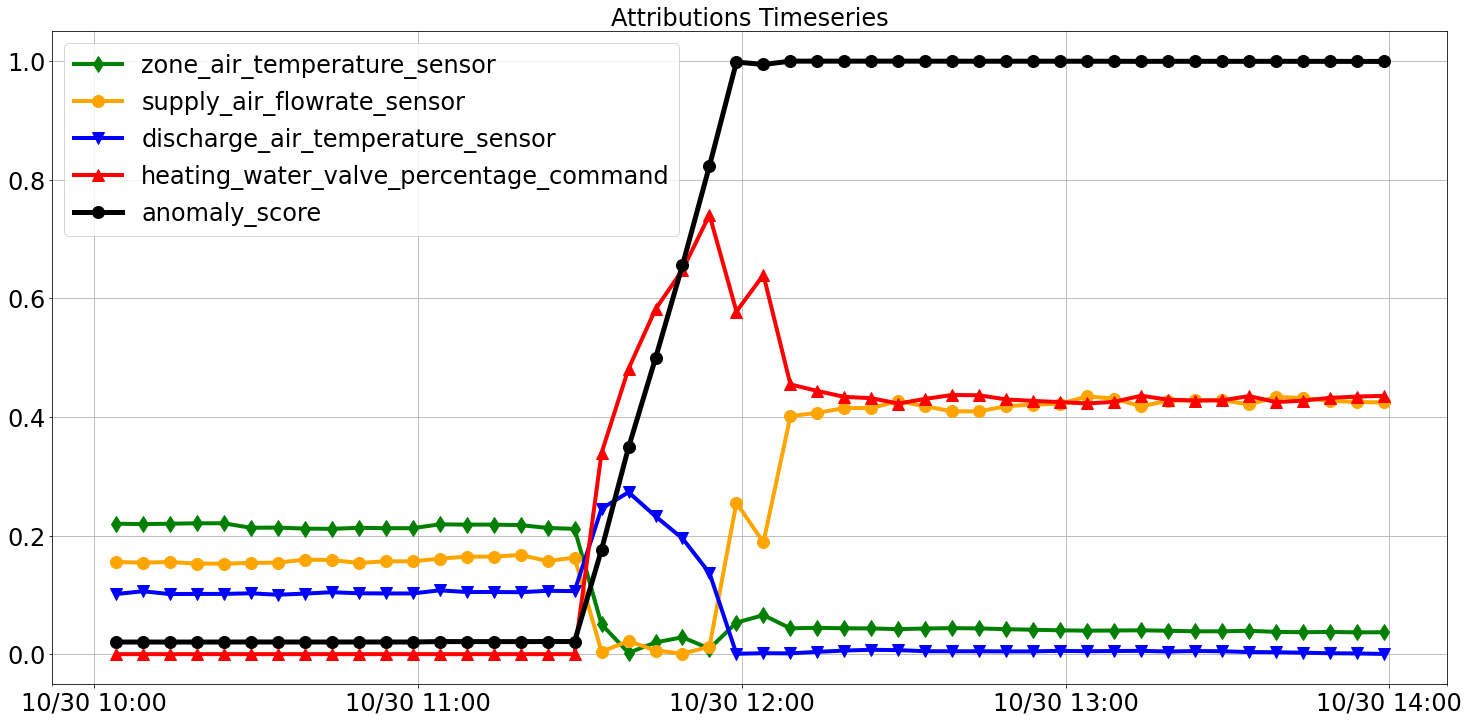}
  \caption{Time series of anomaly scores and blame attributions on an anomalous VAV. The black line represents the anomaly score, $1-F(x)$, where 0 is normal and 1.0 is anomalous (this reverse in convention makes it easier to visualize). The colored lines indicate the attributions for each observation.}
  \label{fig:timeseries}
\end{figure}

A closer inspection shortly after the failure, Figure \ref{fig:blames}, reveals at 11:54 74\% blame is assigned to the heating water valve with an anomalous value of 100\% (expected 50\%). However, at 12:24, the air supply flowrate confounds the root cause by increasing flowrate to 1.2, (expected 0.5) cubic feet per minute, resulting shared blame (43\% and 41\%).

This example illustrates two rules of attributions. First, an insufficient number of independent observable dimensions leads to an ambiguous diagnosis. Second, over time the evolving state of a system may obscure the root cause, making it necessary to attend to the attributions from the earliest anomalous observations.  



\begin{figure}[htb]
  \centering
  \begin{subfigure}
  \centering
    \includegraphics[width=0.51\columnwidth]{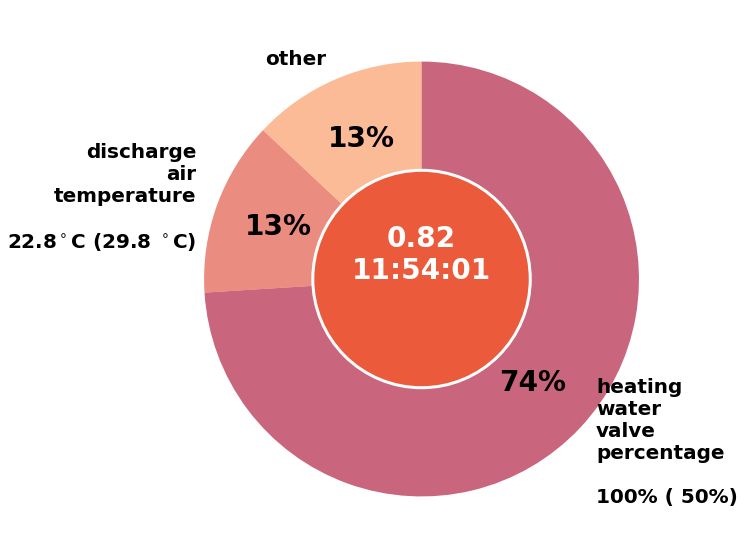}
  \end{subfigure}
  \hfill
  \begin{subfigure}
  \centering
    \includegraphics[width=0.46\columnwidth]{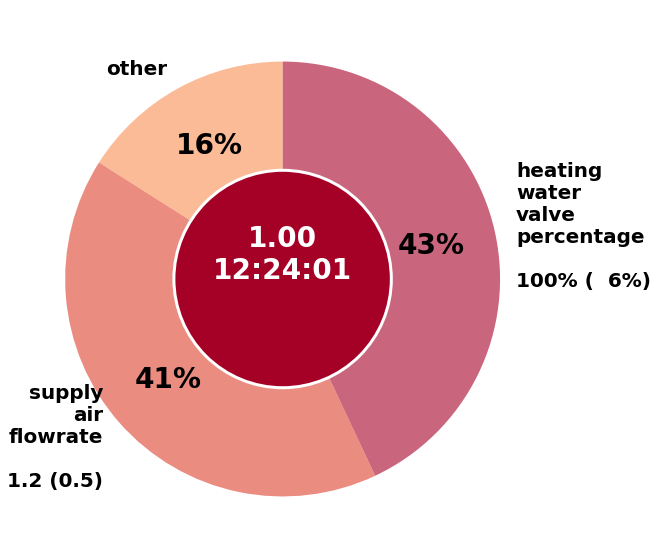}
  \end{subfigure}
  \caption{Blames at 11:54 and 12:24. The anomaly score, $1- F(x)$, and the time are shown in the interior circle, the blame attributions are shown on the outside ring, and the anomalous values is shown next to their baseline values in parenthesis. }
  \label{fig:blames}
\end{figure}

\textbf{Empirical Evaluation}
We evaluated attributions for actual failures on real VAVs. Trained technicians labeled each instance with the subcomponent that was replaced or reconfigured as part of the repair. Each example was labeled a \emph{success} if the attributions correctly identified at least one dimension associated with the fault. For example, if the damper failed, a successful attribution must include air flow set point or measurement. Of the 53 examples, 51 were successful attributions (96\% accuracy). The two unsuccessful were associated with devices that were repaired well after the initial failure occurred.

\end{document}